\documentclass[conference]{IEEEtran}
\IEEEoverridecommandlockouts
\usepackage{cite}
\usepackage{graphicx}
\usepackage{epstopdf}
\usepackage{amsmath,amssymb,amsfonts}
\usepackage{algorithmic}
\usepackage{textcomp}
\usepackage{xcolor}
\usepackage[ruled,lined,linesnumbered]{algorithm2e}
\usepackage{makecell}
\usepackage{extarrows}
\usepackage{multicol}
\usepackage{subcaption}
\usepackage{multirow}
\usepackage{bm}
\usepackage{colortbl}
\usepackage{amsthm}
\usepackage{mathrsfs}
\usepackage{booktabs}
\usepackage{siunitx}
\usepackage{tabularx,booktabs}
\usepackage[flushleft]{threeparttable}
\usepackage[short]{optidef}

\usepackage{xcolor}
\usepackage{xparse}
\usepackage[normalem]{ulem}

\newif\ifshownotes
\shownotesfalse     % final

\newcommand{\defineauthor}[3]{%
  \expandafter\gdef\csname author@#1@name\endcsname{#2}%
  \expandafter\gdef\csname author@#1@color\endcsname{#3}%
}

\NewDocumentCommand{\note}{m m}{%
  \ifshownotes
    \textcolor{\csname author@#1@color\endcsname}
    {\textbf{[\csname author@#1@name\endcsname: #2]}}
  \fi
}

\NewDocumentCommand{\add}{m m}{%
  \ifshownotes
    \textcolor{\csname author@#1@color\endcsname}
    {\textbf{#2}}
  \else
    #2
  \fi
}

\NewDocumentCommand{\del}{m m}{%
  \ifshownotes
    \textcolor{\csname author@#1@color\endcsname}
    {\sout{#2}}
  \fi
}

\NewDocumentCommand{\rep}{m m m}{%
  \ifshownotes
    {\textcolor{\csname author@#1@color\endcsname}{\textbf{#3}}}%
    {\textcolor{\csname author@#1@color\endcsname}{\;{\scriptsize(was: #2)}}}%
  \else
    #3%
  \fi
}
\defineauthor{GLM5}{GLM-5}{purple}
\defineauthor{SK}{Codex}{blue}
\defineauthor{SP}{Supervisor}{red}

\newtheorem{remark}{Remark}

\newtheorem{proposition}{Proposition}

\definecolor{proposedblue}{RGB}{242,242,255}     % light
\definecolor{proposedblueB}{RGB}{255, 220, 225}   % slightly darker

\newcommand{\proposedrow}{\rowcolor{proposedblue}}
\newcommand{\proposedrowB}{\rowcolor{proposedblueB}}
\DeclareMathOperator*{\di}{\mathrm{d}\!}

\def\BibTeX{{\rm B\kern-.05em{\sc i\kern-.025em b}\kern-.08em
T\kern-.1667em\lower.7ex\hbox{E}\kern-.125emX}}

\begin{document}

\title{Generative Semantic Communication via Alternating Dual-Domain Posterior Sampling}

\author{
% \thanks{This work is partly supported by the National Key R\&D Program of China under Grant No. 2024YFE0200802, partly by NSFC under grant No.62293481 and No.62201505, and partly supported by the China Scholarship Council (No. 202406320381) and the CAST Young Talent Support Program for Doctoral Students.}
\IEEEauthorblockN{
Shunpu Tang and 
Qianqian Yang\IEEEauthorrefmark{1},
    }
\IEEEauthorblockA{ 
College of Information Science and Electronic Engineering, Zhejiang University, Hangzhou, China \\
Email: \{tangshunpu,  qianqianyang20\}@zju.edu.cn
}
}

\maketitle

\thispagestyle{empty}
\pagestyle{empty}
\begin{abstract}
  Generative semantic communication (SemCom) harnesses pretrained generative priors to improve the perceptual quality of wireless image transmission. Existing generative SemCom receivers, however, rely on maximum a posteriori (MAP) estimation, which fundamentally cannot preserve the data distribution and thus limits achievable perceptual quality. Moreover, current diffusion-based approaches using single-domain guidance face significant limitations: latent-domain guidance is sensitive to channel noise, while image-domain guidance inherits decoder bias. Simply combining both domains simultaneously yields an \textcolor{black}{overconfident} pseudo-posterior. In this paper, we formulate semantic decoding as a Bayesian inverse problem and prove that posterior sampling achieves optimal perceptual quality by preserving the data distribution. Building on this insight, we propose alternating dual-domain posterior sampling (ADDPS), a diffusion-based SemCom receiver that alternately enforces latent-domain and image-domain consistency during the sampling process. This alternating strategy decomposes joint posterior sampling into simpler subproblems, avoiding gradient conflicts while retaining the complementary strengths of both domains. Experiments on FFHQ demonstrate that the proposed ADDPS achieves superior perceptual quality \textcolor{black}{compared with} existing methods. 
\end{abstract}


\section{Introduction}
Recently, semantic communication (SemCom) has emerged as a transformative alternative to the classical bit-oriented communication paradigm \cite{Semantic1,Semantic2}, aiming to extract and transmit only task-relevant semantic information. The seminal work of deep joint source-channel coding (DeepJSCC) \cite{DeepJSCC} established an end-to-end SemCom framework for wireless image transmission, which has since been extended with advanced architectures \cite{SWINJSCC}, training strategies \cite{Shunpu_SemCom_TCom}, and SNR-adaptive designs \cite{dingSNRAdaptiveDeepJoint2021}.
Despite these advances, existing SemCom systems optimizing end-to-end distortion metrics often produce blurry, over-smoothed reconstructions with poor perceptual quality \cite{blau2018perception,blau2019rethinking}. \textcolor{black}{Fortunately}, this issue can be mitigated by \textit{generative SemCom} \cite{liang2025generative}, which leverages powerful generative priors to reduce the distributional mismatch between reconstructions and the original data. In this direction, early work \cite{GAN_JSCC,tang2024evolving} introduced pretrained GANs at the receiver to reconstruct high-fidelity images via latent inversion. More recently, \emph{generative diffusion models (GDMs)} \cite{ho2020denoising,song2021scorebased} have attracted significant attention due to their superior generation quality and flexible conditional generation, thereby enabling \emph{diffusion-based SemCom} \cite{CDDM_SemCom,yilmaz2023high,tang2025enabling}. Among these, \emph{HiFi-DiffCom} \cite{DiffCom} treats the channel-received signal as a condition for diffusion posterior sampling \cite{chung2022diffusion} and achieves state-of-the-art perceptual quality by further incorporating the decoder output as an auxiliary condition.

However, these generative SemCom methods still suffer significant performance degradation under severe bandwidth compression and very low SNR conditions. Therefore, in this paper, we revisit semantic decoding from a Bayesian perspective. Specifically, most existing approaches \cite{GAN_JSCC,tang2024evolving} formulate receiver-side reconstruction as a maximum a posteriori (MAP) problem. However, MAP is a \emph{point estimator} that collapses the full posterior into a single solution, which cannot preserve the data distribution and limits achievable perceptual quality. This motivates adopting posterior sampling, as adopted in \cite{DiffCom}, which draws from the full posterior and can theoretically achieve optimal perceptual quality. We further identify that existing guidance strategies in diffusion-based SemCom are suboptimal: \emph{latent-domain (Z-domain)} guidance is vulnerable to channel noise at low SNRs, while \emph{image-domain (X-domain)} guidance may discard semantic evidence in the received signal. Although \cite{DiffCom} combines both terms at every step, such a strategy may blur the conditional independence structure and thus tends to produce an overconfident pseudo-posterior.

Motivated by this, we propose \emph{alternating dual-domain posterior sampling (ADDPS)}, which alternates between Z-domain and X-domain guidance during the diffusion sampling process. The main contributions are summarized as follows: (i) we formulate generative semantic decoding as a Bayesian inverse problem and prove that MAP cannot preserve the data distribution, whereas posterior sampling is distributionally optimal; (ii) we propose an alternating dual-domain guidance scheme that exploits the complementary robustness of Z- and X-domain consistency while avoiding the \textcolor{black}{overconfident} pseudo-posterior caused by simultaneous guidance; and (iii) \textcolor{black}{experiments} on FFHQ demonstrate that, under the highly challenging setting of a bandwidth compression ratio of 1/192 and $\mathrm{SNR}=1 \mathrm{~dB}$, the proposed ADDPS achieves superior perceptual quality over existing methods, attaining an FID of 56.94. Ablation studies further validate the effectiveness of the proposed alternating design.
\section{Preliminaries}

\subsection{SemCom for Perceptual Image Transmission}
We consider a SemCom system for wireless image transmission. The system comprises an NN-based semantic encoder at the transmitter and an NN-based semantic decoder at the receiver. Let the source image be $\bm{x} \in \mathbb{R}^{3 \times N_{\text{H}} \times N_{\text{W}}}$, where the first dimension corresponds to the RGB color channels, and $N_{\text{H}}$ and $N_{\text{W}}$ denote the image height and width, respectively. The compression efficiency is characterized by BCR, defined as $\text{BCR} = \frac{k}{N}$, where $k$ is the dimension of the complex-valued channel input signal and $N = 3 \times N_{\text{H}} \times N_{\text{W}}$ is the source bandwidth. At the transmitter, the semantic encoder extracts \textcolor{black}{semantic information} from the source image $\bm{x}$ and maps it to a $k$-dimensional complex-valued channel input signal $\bm{z} \in \mathbb{C}^k$, which can be expressed as
\begin{equation}
  \bm{z} = \mathcal{E}_\theta(\bm{x}),
\end{equation}
where $\mathcal{E}_\theta(\cdot)$ denotes the semantic encoder parameterized by $\theta$. To satisfy the transmit power constraint $P$, a power normalization layer is applied to normalize the transmitted signal $\bm{z}$. For simplicity, we absorb the power normalization operation into the encoder mapping and still use $\bm{z}$ to denote the channel input signal after normalization.

The normalized signal $\bm{z}$ is transmitted over an additive white Gaussian noise (AWGN) channel. The received signal is given by
\begin{equation}
  \hat{\bm{z}} = \bm{z} + \bm{n}_{\text{ch}},
\end{equation}
where $\bm{n}_{\text{ch}} \sim \mathcal{CN}(0, \sigma_{\text{ch}}^2 \bm{I})$ is the channel noise with zero mean and variance $\sigma_{\text{ch}}^2$. The signal-to-noise ratio (SNR) can hence be defined as $\text{SNR} = \frac{P}{\sigma_{\text{ch}}^2}$. After receiving the corrupted signal $\hat{\bm{z}}$, the receiver employs the semantic decoder to reconstruct the source image, given by
\begin{equation}
  \tilde{\bm{x}} = \mathcal{D}_{\phi}(\hat{\bm{z}}),
\end{equation}
where $\mathcal{D}_\phi(\cdot)$ denotes the semantic decoder \textcolor{black}{parameterized by} $\phi$ and $\tilde{\bm{x}}$ is the reconstructed image. We note that recent advances have incorporated generative models to further enhance the perceptual quality of the reconstruction. In such cases, the decoder output $\tilde{\bm{x}}$ or the received signal $\hat{\bm{z}}$ serves as a conditioning signal, yielding the final reconstruction:
\begin{equation}
  \hat{\bm{x}} \sim \mathcal{G}_\psi(\tilde{\bm{x}}, \hat{\bm{z}}),
\end{equation}
where $\mathcal{G}_\psi(\cdot)$ denotes the generative model parameterized by $\psi$ and $\hat{\bm{x}}$ is the final reconstructed image.

To evaluate reconstruction quality, distortion metrics such as PSNR and MS-SSIM are commonly employed. Following \cite{blau2018perception,blau2019rethinking}, we also consider perceptual quality, which is quantified by the discrepancy between the reconstruction distribution $p_{\hat{\bm{x}}}$ and the original data distribution $p_{\mathrm{data}}$, and can be practically measured by the Fr\'echet Inception distance (FID).

\subsection{Generative Diffusion Models}
GDMs provide a powerful framework for learning complex data distributions and have achieved remarkable success in image synthesis. A GDM consists of a forward noising process that gradually corrupts data into noise, and a reverse denoising process that learns to recover data from noise. Formally, the forward process for data $\bm{x}_0 \sim p_{\text{data}}$ is described by the stochastic differential equation (SDE)
\begin{equation}
  \label{eq:forward_sde}
  \di \bm{x}_t = -\frac{1}{2}\beta(t)\bm{x}_t\di t + \sqrt{\beta(t)}\di \bm{w}_t,
\end{equation}
where $t \in [0, 1]$, $\beta(t)$ is a noise schedule, and $\bm{w}_t$ is a standard Wiener process. The corresponding reverse process, which starts from a Gaussian prior and generates samples of $\bm{x}_0$, follows the reverse-time SDE
\begin{equation}
  \label{eq:reverse_sde}
  \di \bm{x}_t = \left(-\frac{1}{2}\beta(t)\bm{x}_t - \beta(t)\nabla_{\bm{x}_t} \log p_t(\bm{x}_t)\right)\di t + \sqrt{\beta(t)}\di \bar{\bm{w}}_t,
\end{equation}
where $\bar{\bm{w}}_t$ is a reverse-time Wiener process. Solving this SDE requires the score function $\nabla_{\bm{x}_t}\log p_t(\bm{x}_t)$, which is approximated by a neural network $\bm{s}_\omega(\bm{x}_t,t)$ trained via denoising score matching~\cite{Vincent2011ConnectionScoreMatching} with the objective
\begin{equation}
  \omega^* = \arg\min_\omega \mathbb{E}_{t, \bm{x}_t, \bm{x}_0}\left[\|\bm{s}_\omega(\bm{x}_t, t) - \nabla_{\bm{x}_t}\log p(\bm{x}_t|\bm{x}_0)\|_2^2\right].
\end{equation}
Throughout this paper, we adopt the variance-preserving (VP) SDE formulation \cite{song2021scorebased}, which is equivalent to the denoising diffusion probabilistic model (DDPM) \cite{ho2020denoising} in discrete time. We define discrete-time notations with $T$ steps: $\bm{x}_i \triangleq \bm{x}(i/T)$, $\alpha_i \triangleq 1 - \beta_i$, and $\bar{\alpha}_i \triangleq \prod_{j=1}^{i} \alpha_j$. Under this notation, the forward process becomes
\begin{equation}
  \bm{x}_i = \sqrt{\bar{\alpha}_i}\bm{x}_0 + \sqrt{1-\bar{\alpha}_i}\bm{\epsilon}, \quad \bm{\epsilon} \sim \mathcal{N}(0, \bm{I}),
\end{equation}
and the reverse process is solved iteratively as
\begin{equation}
  \bm{x}_{i-1} = \frac{1}{\sqrt{\alpha_i}}\left(\bm{x}_i + \beta_i\,\bm{s}_\omega(\bm{x}_i, i)\right) + \sqrt{\beta_i}\,\bm{\epsilon}.
\end{equation}

% We note that the generative process can be conditioned on auxiliary information $\bm{y}$ by sampling from $p(\bm{x}|\bm{y})$. Following Bayes' rule, the conditional score decomposes as
% \begin{equation}
%   \nabla_{\bm{x}_t} \log p_t(\bm{x}_t|\bm{y}) = \nabla_{\bm{x}_t} \log p_t(\bm{x}_t) + \nabla_{\bm{x}_t} \log p_t(\bm{y}|\bm{x}_t),
% \label{eq:conditional_score}
% \end{equation}
% where the first term is the prior score and the second term is the likelihood score.
% However, for solving inverse problems,  while the prior score $\nabla{\bm{x}_t} \log p_t(\bm{x}_t)$ can be approximated by the score network, the likelihood term $\nabla{\bm{x}_t} \log p_t(\bm{y}|\bm{x}_t)$ is generally intractable due to the absence of an explicit dependency between $\bm{y}$ and the noisy intermediate $\bm{x}_t$. Diffusion posterior sampling (DPS) \cite{chung2022diffusion} addresses this challenge by approximating the likelihood via Tweedie's formula: $p(\bm{y}|\bm{x}_t) \approx p(\bm{y}|\hat{\bm{x}}_0)$, where $\hat{\bm{x}}_0$ is a denoised estimate of $\bm{x}_t$, thereby enabling practical conditional generation.
\begin{figure*}
  \centering
  \vspace{0.08in}
  \includegraphics[width=0.68\textwidth]{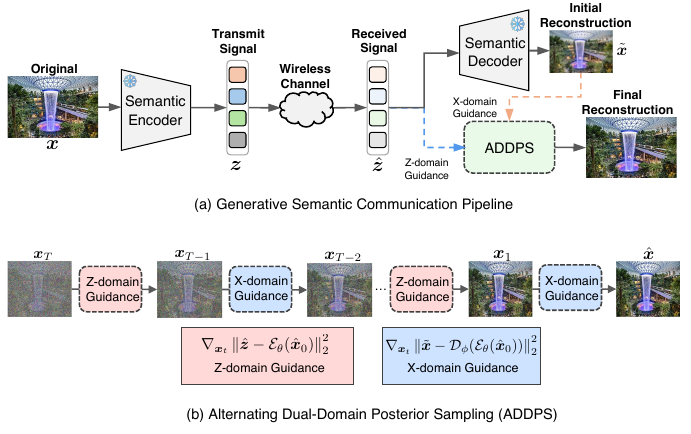}
  \caption{Illustration of the proposed ADDPS method, where the reverse diffusion process alternates between \textcolor{black}{Z-domain and X-domain guidance to enforce consistency in both domains while avoiding noise amplification in Z-domain guidance.}}
  \label{fig:addps}
\end{figure*}
\section{Proposed Method}
In this section, we present the proposed diffusion-based SemCom framework. We first formulate the semantic decoding process as an inverse problem from a Bayesian perspective and establish dual-domain measurement models. We then introduce the proposed ADDPS method.
\subsection{Generative Semantic Decoding as \textcolor{black}{an} Inverse Problem}
In generative SemCom systems, the transmission process can be viewed as a degradation process comprising semantic compression and channel corruption. Reconstructing the source image $\bm{x}$ from the noisy received signal $\hat{\bm{z}}$ is therefore an \emph{inverse problem}. A natural Bayesian approach is to find the most probable reconstruction given the observation, i.e., \textcolor{black}{the} MAP estimate
\begin{equation}
\hat{\bm{x}}_{\mathrm{MAP}} = \arg \max_{\bm{x}}\; p(\bm{x} | \hat{\bm{z}}).
\end{equation}
By Bayes' rule, the posterior distribution can be written as
\begin{equation}
p(\bm{x}| \hat{\bm{z}}) \propto p(\hat{\bm{z}}| \bm{x})\, p_{G}(\bm{x}),
\label{eq:posterior_factorization}
\end{equation}
where $p_{G}(\bm{x})$ serves as the prior and $p(\hat{\bm{z}}| \bm{x})$ is the likelihood induced by the semantic encoder and the wireless channel. For a deterministic semantic encoder and an AWGN channel with noise variance $\sigma_{\text{ch}}^2$, the likelihood can be written as
\begin{equation}
p(\hat{\bm{z}} | \bm{x}) \propto 
\exp\!\left( -\frac{\| \hat{\bm{z}} - \mathcal{E}_\theta(\bm{x}) \|_2^2}{2\sigma_{\mathrm{ch}}^2} \right),
\label{eq:likelihood}
\end{equation}
which penalizes the mismatch between $\hat{\bm{z}}$ and $\mathcal{E}_\theta(\bm{x})$, thereby enforcing semantic consistency \cite{Shunpu_SemCom_TCom}. Taking the negative logarithm of \eqref{eq:posterior_factorization} and dropping constants, the MAP problem becomes
\begin{equation}
\hat{\bm{x}}_{\mathrm{MAP}}
=\arg\min_{\bm{x}} \;\|\hat{\bm{z}}-\mathcal{E}_\theta(\bm{x})\|_2^2 - \log p_{G}(\bm{x}),
\label{eq:map_decoder}
\end{equation}
where the first term is the semantic consistency loss and the second term is the prior regularizer. This MAP formulation is adopted, either explicitly or implicitly, by existing generative SemCom methods such as \cite{GAN_JSCC,tang2024evolving}. However, the MAP estimate is inherently a \emph{point estimator} that collapses the full posterior into a single mode, which leads to a mismatch between the reconstruction distribution and the original data distribution, thereby limiting the achievable perceptual quality.
\begin{proposition}
\label{prop:map_mismatch}
Solving the MAP problem \eqref{eq:map_decoder} does not, in general, preserve the data distribution, i.e.,
$p_{\hat{\bm{x}}}(\bm{x}) \neq p_{\mathrm{data}}(\bm{x})$.
\end{proposition}
\begin{proof}[Proof sketch]
Consider a scalar Gaussian source $x\sim\mathcal{N}(0,\sigma_x^2)$ with identity encoder and AWGN channel noise $\sigma_n^2>0$. By Bayes' rule, $\hat{x}_{\mathrm{MAP}}=\alpha \hat{z}$ with $\alpha=\sigma_x^2/(\sigma_x^2+\sigma_n^2)<1$, yielding $\mathrm{Var}(\hat{x}_{\mathrm{MAP}})=\sigma_x^4/(\sigma_x^2+\sigma_n^2) < \sigma_x^2$. Hence the MAP output distribution is strictly narrower than the source. 
\end{proof}
\begin{remark}
  \textcolor{black}{The scalar case shows} that this mismatch is caused by channel noise. When $\sigma_n^2 = 0$, we have $\hat{z} = x$, so the received signal follows the same distribution as the source. When $\sigma_n^2 > 0$, however, the receiver only observes a noisy version $\hat{z}$. The MAP estimate is a deterministic function of $\hat{z}$, so it cannot reproduce the full randomness of the source distribution, which limits perceptual quality~\cite{blau2018perception,blau2019rethinking}. 
\end{remark}
% sub figures *3

This distribution mismatch motivates us to reconsider the decoding paradigm. Instead of seeking point estimates that minimize reconstruction error, we turn to \emph{posterior sampling}, which draws samples from $p(\bm{x} | \hat{\bm{z}})$ to better capture the uncertainty in the source.
\begin{proposition}
\label{prop:perceptual_optimality}
\textcolor{black}{At the receiver}, if we can sample from the posterior $p(\bm{x} | \hat{\bm{z}})$, then the marginal distribution matches the data distribution $p_{\hat{\bm{x}}}(\bm{x}) = p_{\mathrm{data}}(\bm{x})$ when the generative model captures the data distribution exactly.
\end{proposition}
\begin{proof}
By the law of total probability and Bayes' rule, we have
$p_{\hat{\bm{x}}}(\bm{x}) = \int p(\bm{x} |\hat{\bm{z}}) p(\hat{\bm{z}})\mathrm{d}\hat{\bm{z}} = p_{\mathrm{data}}(\bm{x}) \int p(\hat{\bm{z}} |\bm{x})\mathrm{d}\hat{\bm{z}} = p_{\mathrm{data}}(\bm{x})$.
\end{proof}

Propositions~\ref{prop:map_mismatch} and~\ref{prop:perceptual_optimality} together establish that generative semantic decoding should be viewed as a \emph{posterior sampling problem} rather than an optimization problem.  Specifically, rather than sampling from the unconditional distribution $p(\bm{x})$ as in \eqref{eq:reverse_sde}, we can replace the prior score $\nabla_{\bm{x}_t}\log p_t(\bm{x}_t)$ with the posterior score $\nabla_{\bm{x}_t}\log p_t(\bm{x}_t|\hat{\bm{z}})$, thereby guiding the reverse diffusion process with the received signal $\hat{\bm{z}}$, given by
\begin{equation}
\di \bm{x}_t =
\left(-\frac{1}{2}\beta(t)\bm{x}_t - \beta(t)\nabla_{\bm{x}_t}\log p_t(\bm{x}_t|\hat{\bm{z}})\right)\di t
 + \sqrt{\beta(t)}\di \bar{\bm{w}}_t .
\label{eq:dps_reverse_sde}
\end{equation}
\textcolor{black}{Using} Bayes' rule, we can further expand the posterior score as
\begin{equation}
\nabla_{\bm{x}_t}\log p_t(\bm{x}_t|\hat{\bm{z}})
= \nabla_{\bm{x}_t}\log p_t(\bm{x}_t) + \nabla_{\bm{x}_t}\log p_t(\hat{\bm{z}}|\bm{x}_t),
\label{eq:dps_posterior_score}
\end{equation}
where the first term is the prior score and the second term is the likelihood score. While the prior score is available from the pretrained diffusion model, the likelihood score $\nabla_{\bm{x}_t}\log p_t(\hat{\bm{z}}|\bm{x}_t)$ is intractable due to the absence of an explicit dependency between $\hat{\bm{z}}$ and the noisy intermediate $\bm{x}_t$. In other words, we only have an explicit likelihood model $p(\hat{\bm{z}}|\bm{x})$ defined on the clean image $\bm{x}$.
To see this, we can further expand the likelihood term as
\begin{equation}
  \begin{aligned}
    p_t(\hat{\bm{z}}|\bm{x}_t) = \int p(\hat{\bm{z}}|\bm{x}_0) p(\bm{x}_0|\bm{x}_t)\mathrm{d}\bm{x}_0, 
  \end{aligned}
\end{equation}
where the posterior $p(\bm{x}_0|\bm{x}_t)$ has no closed-form expression. Following \cite{chung2022diffusion,DiffCom}, we approximate the above expectation by evaluating the likelihood at the posterior mean of $\bm{x}_0$ given $\bm{x}_t$, which \textcolor{black}{is} given by
\begin{equation}
  p_t(\hat{\bm{z}}|\bm{x}_t)
  \approx
  p(\hat{\bm{z}}|\hat{\bm{x}}_{0,t}),
\end{equation}
where
\begin{equation}
\hat{\bm{x}}_{0,t} \triangleq \mathbb{E}[\bm{x}_0|\bm{x}_t]
\simeq \frac{1}{\sqrt{\bar{\alpha}_t}}\bigg(\bm{x}_t + (1-\bar{\alpha}_t)\,\bm{s}_\omega(\bm{x}_t,t)\bigg),
\label{eq:tweedie_estimate}
\end{equation}
denotes the posterior mean estimator, yielding the approximate gradient
\begin{equation}
\nabla_{\bm{x}_t}\log p_t(\hat{\bm{z}}|\bm{x}_t)\ \approx\ \nabla_{\bm{x}_t}\log p(\hat{\bm{z}}|\hat{\bm{x}}_{0,t}).
\label{eq:dps_likelihood_approx}
\end{equation}
Substituting the Gaussian likelihood \eqref{eq:likelihood} and the decomposition \eqref{eq:dps_posterior_score} into the reverse SDE, the overall posterior score is approximated as
\begin{equation}
  \nabla_{\bm{x}_t}\log p_t(\bm{x}_t|\hat{\bm{z}})
\simeq
\bm{s}_\omega(\bm{x}_t,t) -\rho_{Z}\, \nabla_{\bm{x}_t} \|\hat{\bm{z}}-\mathcal{E}_\theta(\hat{\bm{x}}_{0,t})\|_2^2,
\label{eq:z_guidance_gradient}
\end{equation}
where $\rho_{Z} > 0$ is a guidance strength parameter. By iterating the reverse SDE \eqref{eq:dps_reverse_sde} with this approximation, we can draw samples from the posterior $p(\bm{x} | \hat{\bm{z}})$ and obtain the reconstructed image.
\subsection{Dual-Domain Consistency Formulation}
Based on the posterior sampling formulation above, we now discuss why using only latent consistency is insufficient, especially under severe bandwidth compression and low SNR. Since $\hat{\bm{z}} = \bm{z} + \bm{n}_{\mathrm{ch}}$, the Z-domain gradient in \eqref{eq:z_guidance_gradient} contains a channel-noise term that becomes dominant at low SNR, degrading semantic consistency. This is also empirically observed in \cite{DiffCom}: lowering the guidance strength yields unfaithful reconstructions, whereas increasing it amplifies channel noise and introduces artifacts.

To overcome this limitation, \textcolor{black}{we introduce} an additional guidance term by using the paired semantic decoder $\mathcal{D}_\phi$. Specifically, \textcolor{black}{we regard} the decoder output $\tilde{\bm{x}}$ as a noisy measurement of the original image $\bm{x}$ to enforce \textcolor{black}{image-domain} consistency. In this case, \textcolor{black}{we express} the X-domain likelihood as
\begin{equation}
  p(\tilde{\bm{x}} | \bm{x}) \propto \exp\left( -\frac{\| \tilde{\bm{x}} - \mathcal{D}_\phi(\mathcal{E}_\theta(\bm{x})) \|_2^2}{2\sigma_{\text{X}}^2} \right),
  \label{eq:x_likelihood}
\end{equation}
where $\sigma_{\text{X}}^2$ captures the effective noise in the image domain. Then, similar to the Z-domain case, we can still use DPS to approximate the posterior score for X-domain guidance, given by
\begin{equation}
  \nabla_{\bm{x}_t}\log p_t(\bm{x}_t|\tilde{\bm{x}}) \simeq
\bm{s}_\omega(\bm{x}_t,t) -\rho_{X}\nabla_{\bm{x}_t}\|\tilde{\bm{x}}-\mathcal{D}_\phi(\mathcal{E}_\theta(\hat{\bm{x}}_{0,t}))\|_2^2,
\end{equation}
where $\rho_{X} > 0$ is the guidance strength for X-domain consistency. Intuitively, a well-trained decoder $\mathcal{D}_\phi$ is robust to channel noise, so the effective noise in the X-domain gradient, $\mathcal{D}_\phi(\hat{\bm{z}})-\mathcal{D}_\phi(\bm{z})$, is much smaller than the raw channel noise in the Z-domain. Hence X-domain consistency is more robust than Z-domain consistency under low-SNR channel perturbations.
\begin{table*}[!t]
    \centering
    \vspace{0.10in}
    \setlength{\tabcolsep}{12pt}
    \renewcommand{\arraystretch}{1.2}
    \caption{Comparison of different methods at $\mathrm{SNR}=1$ dB.}
    \label{tab:snr_pos1_comparison}
      \begin{tabular}{lcccccc}
        \toprule
        Method & FID$\downarrow$ & PIEAPP$\downarrow$ & DISTS$\downarrow$ & LPIPS$\downarrow$ & PSNR (dB)$\uparrow$ & MS-SSIM$\uparrow$ \\
        \midrule
        DeepJSCC \cite{DeepJSCC}& 108.62 & 2.1246 & 0.2782 & 0.3532 & \textbf{24.08} & \textbf{0.8685} \\
        DeepJSCC-LPIPS & 59.16 & 1.9739 & \textbf{0.1546} & \textbf{0.1903} & 21.98 & 0.8323 \\
        \midrule
        InverseJSCC \cite{GAN_JSCC} & 67.70 & 1.3191 & 0.1764 & 0.2287 & 19.05 & 0.7314 \\
        HiFi-DiffCom ($T$=300)~\cite{DiffCom} & 60.08 & 1.3230 & 0.1697 & 0.2022 & 23.35 & \underline{0.8479} \\
        HiFi-DiffCom ($T$=1000)~\cite{DiffCom} & 59.54 & \underline{1.2977} & 0.1640 & 0.1978 & \underline{23.80} & 0.8411 \\
        \midrule
        \proposedrow
        Proposed ($T$=300) & \underline{58.68} & 1.3348 & 0.1681 & 0.2066 & 22.77 & 0.8345 \\
        \proposedrowB
        Proposed ($T$=1000) & \textbf{56.94} & \textbf{1.2930} & \underline{0.1600} & \underline{0.1950} & 22.99 & 0.8391 \\
        \bottomrule
      \end{tabular}
  \end{table*}

 \begin{algorithm}[t]
  \caption{Proposed ADDPS}
  \label{alg:addps}
  \SetKwInOut{Input}{Input}
  \SetKwInOut{Output}{Output}
  \Input{Received signal $\hat{\bm{z}}$, semantic encoder $\mathcal{E}_\theta$, semantic decoder $\mathcal{D}_\phi$, score network $\bm{s}_\omega$, total steps $T$, guidance strengths $\{\zeta_t\}$}
  \Output{Reconstructed image $\hat{\bm{x}}$}
  \BlankLine
  \textcolor{black}{$\tilde{\bm{x}} \leftarrow \mathcal{D}_\phi(\hat{\bm{z}})$}
  
  \textcolor{black}{$\bm{x}_T \sim \mathcal{N}\!\left(\sqrt{\bar{\alpha}_T}\tilde{\bm{x}},\ (1-\bar{\alpha}_T)\bm{I}\right)$} \\
  \For{$t = T, T-1, \ldots, 1$}{
    \textcolor{black}{Predict clean image via Tweedie: $\hat{\bm{x}}_0 \leftarrow \frac{1}{\sqrt{\bar{\alpha}_t}}(\bm{x}_t + (1-\bar{\alpha}_t)\bm{s}_\omega(\bm{x}_t, t))$} \\
    \textcolor{black}{$\bm{x}'_{t-1} \sim \mathcal{N}\!\left(\frac{1}{\sqrt{\alpha_t}}\!\left(\bm{x}_t + (1-\alpha_t)\bm{s}_\omega(\bm{x}_t,t)\right),\ \sigma_t^2\bm{I}\right)$} \\
    \eIf{$t \mod 2 = 0$}{
      \tcp{\textcolor{blue}{Z-Domain Guidance}}
      \textcolor{blue}{$\bm{g}_t \leftarrow \nabla_{\bm{x}_t} \left\| \hat{\bm{z}} - \mathcal{E}_\theta(\hat{\bm{x}}_0) \right\|_2^2$}
    }{
      \tcp{\textcolor{blue}{X-Domain Guidance}}
      \textcolor{blue}{$\bm{g}_t \leftarrow \nabla_{\bm{x}_t} \left\| \tilde{\bm{x}} - \mathcal{D}_\phi(\mathcal{E}_\theta(\hat{\bm{x}}_0)) \right\|_2^2$}
    }
    
    \textcolor{black}{$\bm{x}_{t-1} \leftarrow \bm{x}'_{t-1} - \zeta_t \cdot \bm{g}_t$}
  }
  \Return{$\hat{\bm{x}} = \bm{x}_0$}
\end{algorithm}

Despite the robustness of X-domain consistency, using it alone is insufficient. Specifically, sampling only from $p(\bm{x}| \tilde{\bm{x}})\propto p(\bm{x})p(\tilde{\bm{x}}|\bm{x})$ conditions on a decoder-generated statistic rather than the original channel observation. Since $\tilde{\bm{x}}=\mathcal{D}_{\phi}(\hat{\bm{z}})$ is a deterministic projection, it may discard part of the semantic evidence in $\hat{\bm{z}}$ and inherit decoder bias, i.e., different latent-consistent candidates can map to similar $\tilde{\bm{x}}$. 

To retain latent-domain fidelity while preserving image-domain robustness, we model reconstruction with the joint posterior distribution $p(\bm{x} | \hat{\bm{z}}, \tilde{\bm{x}})$, which enforces both constraints simultaneously and reduces the ambiguity of X-only conditioning. 

\subsection{Alternating Dual-Domain Posterior Sampling}
To sample from the joint posterior distribution, according to \cite{DiffCom}, a straightforward approach is to rewrite \eqref{eq:dps_posterior_score} as
\begin{equation}
  \begin{aligned}
    \label{eq:joint_posterior_bayes}
  & \nabla_{\bm{x}_t} \log p(\bm{x}_{t} | \hat{\bm{z}}, \tilde{\bm{x}}) \\
  &= \nabla_{\bm{x}_t} \log p(\bm{x}_{t}) + \nabla_{\bm{x}_t} \log p(\hat{\bm{z}}, \tilde{\bm{x}} | \bm{x}_{t}) \\
  & \approx \bm{s}_\omega(\bm{x}_t,t) + \nabla_{\bm{x}_t} \log p(\hat{\bm{z}}|\bm{x}_t) + \nabla_{\bm{x}_t} \log p(\tilde{\bm{x}}|\bm{x}_t)
  \end{aligned}
  \end{equation}
  and use DPS \cite{chung2022diffusion} to \textcolor{black}{perform sampling by replacing $\bm{x}_t$ with $\hat{\bm{x}}_{0,t}$}. However, directly applying simultaneous dual-domain guidance is not well justified. Specifically, \eqref{eq:joint_posterior_bayes} implicitly treats \(\hat{\bm{z}}\) and \(\tilde{\bm{x}}\) as conditionally independent given \(\bm{x}\), which is violated when \(\tilde{\bm{x}}=\mathcal{D}_{\phi}(\hat{\bm{z}})\) is a deterministic transform of \(\hat{\bm{z}}\). Therefore, directly multiplying \(p(\hat{\bm{z}}|\bm{x})\) and \(p(\tilde{\bm{x}}|\bm{x})\) repeatedly uses correlated evidence and tends to produce an \textcolor{black}{overconfident} pseudo-posterior. Moreover, both likelihood guidance terms are computed through the same Tweedie surrogate \(\hat{\bm{x}}_{0,t}\), so approximation bias in \(\hat{\bm{x}}_0\) is injected into both gradients simultaneously, which degrades the stability of the guidance direction.

% \subsubsection{Informed Initialization via Forward Diffusion}
% Unlike standard DPS methods that initialize from pure Gaussian noise $\bm{x}_T \sim \mathcal{N}(\bm{0}, \bm{I})$, we leverage the semantic decoder output $\tilde{\bm{x}} = \mathcal{D}_\phi(\hat{\bm{z}})$ to construct an informed starting point. As established in Remark~\ref{rem:noise_reduction}, the MSE-trained decoder produces a reconstruction whose per-pixel error $\sigma_{\mathrm{X}}^2$ is orders of magnitude smaller than the channel noise variance, so $\tilde{\bm{x}}$ already faithfully captures the semantic content of the transmitted image. Specifically, we initialize the reverse process at timestep $T$ by adding noise to $\tilde{\bm{x}}$ via forward diffusion:
% \begin{equation}
%   \bm{x}_T = \sqrt{\bar{\alpha}_T}  \tilde{\bm{x}} + \sqrt{1 - \bar{\alpha}_T}  \bm{\epsilon}, \quad \bm{\epsilon} \sim \mathcal{N}(\bm{0}, \bm{I}),
%   \label{eq:informed_init}
% \end{equation}
% where $\bar{\alpha}_T = \prod_{s=1}^{T} \alpha_s$ is the cumulative noise schedule coefficient. This informed initialization provides a substantially better starting point than random noise: the reverse diffusion process only needs to refine fine details and textures rather than reconstruct the global structure from scratch, effectively reducing the search space for posterior sampling.
% The remaining question is how to optimize this joint posterior during reverse diffusion. Directly combining the two guidance terms at every step may cause gradient conflict; therefore, we adopt an alternating guidance schedule.

To address this issue, we propose an alternating dual-domain guidance strategy, which is motivated by the alternating optimization framework.
As shown in Fig.~\ref{fig:addps}, rather than applying both guidance terms simultaneously, we schedule Z-domain and X-domain guidance at different diffusion steps, given by
\begin{equation}
  \nabla_{\bm{x}_t}\log p_t^{\mathrm{alt}}(\bm{x}_t|\hat{\bm{z}},\tilde{\bm{x}})
  \approx
  \begin{cases}
    \bm{s}_\omega(\bm{x}_t,t)-\zeta_t\bm{g}_{Z,t}, & t \bmod 2 = 0,\\
    \bm{s}_\omega(\bm{x}_t,t)-\zeta_t\bm{g}_{X,t}, & t \bmod 2 = 1,
  \end{cases}
  \label{eq:alt_guidance}
\end{equation}
where
\begin{equation}
  \begin{aligned}
    \bm{g}_{Z,t}\triangleq\nabla_{\bm{x}_t}\left\|\hat{\bm{z}}-\mathcal{E}_\theta\!\left(\hat{\bm{x}}_{0,t}\right)\right\|_2^2,\\
    \bm{g}_{X,t}\triangleq\nabla_{\bm{x}_t}\left\|\tilde{\bm{x}}-\mathcal{D}_\phi\!\left(\mathcal{E}_\theta\!\left(\hat{\bm{x}}_{0,t}\right)\right)\right\|_2^2.   
  \end{aligned}
\end{equation}
\textcolor{black}{Here, $\zeta_t>0$ denotes the timestep-dependent guidance strength used in the alternating sampling process; it plays the role of the domain-specific coefficients $\rho_Z$ and $\rho_X$ when the current step applies Z-domain and X-domain guidance, respectively.}

The alternating schedule offers several advantages: (1) it avoids gradient conflicts between the two consistency terms that may point in different directions; (2) it allows each domain to fully correct the intermediate estimate before the other domain refines it. We summarize the proposed ADDPS method in Algorithm~\ref{alg:addps}.

\begin{table*}[!t]
  \centering
  \vspace{0.10in}
  \small
  \setlength{\tabcolsep}{8pt}
  \renewcommand{\arraystretch}{0.95}
  \caption{Ablation Study of the Proposed Method}
  \label{tab:ablation_study}
    \begin{tabular}{lccccccc}
      \toprule
      \multirow{2}{*}{Domain} & \multirow{2}{*}{Gradient} & \multicolumn{6}{c}{Metrics} \\
      \cmidrule(lr){3-8}
      & & FID$\downarrow$ & PIEAPP$\downarrow$ & DISTS$\downarrow$ & LPIPS$\downarrow$ & PSNR (dB)$\uparrow$ & MS-SSIM$\uparrow$ \\
      \midrule
      X & $\bm{g}_{X,t}=\nabla_{\bm{x}_t}\|\tilde{\bm{x}}-\mathcal{D}_\phi(\mathcal{E}_\theta(\hat{\bm{x}}_{0,t}))\|_2^2$ & \underline{58.99} & 1.4235 & 0.1779 & 0.2297 & \underline{22.06} & 0.8027 \\
      Z & $\bm{g}_{Z,t}=\nabla_{\bm{x}_t}\|\hat{\bm{z}}-\mathcal{E}_\theta(\hat{\bm{x}}_{0,t})\|_2^2$ & 87.03 & 1.9145 & 0.2931 & 0.4698 & 10.47 & 0.4698 \\
      X+Z & $\bm{g}_t=\bm{g}_{Z,t}+\bm{g}_{X,t}$ & 61.64 & \underline{1.3295} & \underline{0.1725} & \underline{0.2216} & \textbf{22.13} & \textbf{0.8081} \\
      \midrule
      \proposedrowB
      Proposed & $\bm{g}_t=\left\{\begin{aligned}\bm{g}_{Z,t},\ &t\bmod2=0,\\ \bm{g}_{X,t},\ &t\bmod2=1,\end{aligned}\right.$ & \textbf{58.38} & \textbf{1.2976} & \textbf{0.1689} & \textbf{0.2199} & 22.05 & \underline{0.8044} \\
      \bottomrule
  \end{tabular}
\end{table*}

\section{Simulation Results}
\label{Sec:simulations}
In this section, we evaluate the proposed ADDPS method. We first describe the simulation setup, and then report quantitative comparisons, followed by an ablation study.

\subsection{Simulation \textcolor{black}{Settings}}

Simulations are conducted on the FFHQ dataset with \textcolor{black}{the} image resolution fixed to $256 \times 256$. Following \cite{GAN_JSCC}, we adopt the same semantic encoder-decoder backbone to ensure architectural consistency with prior SemCom studies. The bandwidth compression ratio is fixed to $1/192$. For the generative receiver, we use \textcolor{black}{a diffusion} model trained on FFHQ dataset. We consider an AWGN channel and evaluate low-SNR conditions with $\mathrm{SNR}$ \textcolor{black}{set to} 1 dB. We compare ADDPS with four baselines: \textit{DeepJSCC} \cite{DeepJSCC,GAN_JSCC} trained with MSE loss, \textit{DeepJSCC-LPIPS} trained with the LPIPS perceptual loss \cite{LPIPS}, \textit{InverseJSCC} \cite{GAN_JSCC} using GAN inversion at the receiver, and \textit{HiFi-DiffCom} \cite{yilmaz2023high}, the state-of-the-art diffusion-based generative SemCom \textcolor{black}{method} and the most relevant prior work.
For evaluation, we adopt PSNR and MS-SSIM for distortion fidelity, and FID, PIEAPP, DISTS, and LPIPS for perceptual quality.

\subsection{Performance Comparison}
Table~\ref{tab:snr_pos1_comparison} compares the performance of different methods at $\mathrm{SNR}=1$ dB. ADDPS with $T=1000$ achieves the best perceptual quality, with the lowest FID of $56.94$ and PIEAPP of $1.293$, and the second-best DISTS and LPIPS. Compared with \textcolor{black}{the most relevant baseline, HiFi-DiffCom}, ADDPS shows a clear improvement in distribution-level realism; even with $T=300$, ADDPS still surpasses HiFi-DiffCom with $T=1000$. Although DeepJSCC \textcolor{black}{achieves} the best PSNR/MS-SSIM, its FID is nearly twice that of ADDPS, indicating severe over-smoothing. DeepJSCC-LPIPS improves perceptual scores but requires retraining the encoder-decoder with a perceptual loss, whereas ADDPS operates as a plug-in receiver using a pretrained diffusion prior. Overall, this result demonstrates the effectiveness of the proposed ADDPS method in improving the perceptual quality of SemCom under low-SNR conditions.

To further validate the effectiveness of the proposed ADDPS method, we conduct an ablation study to analyze the contribution of each guidance component, where the channel SNR is set to $-1$ dB and the number of diffusion steps is set to $T=1000$. Specifically, we compare four guidance variants: (i) X-domain guidance only, which enforces consistency solely in the image domain; (ii) Z-domain guidance only, which enforces consistency solely in the latent domain; (iii) simultaneous X+Z guidance, which applies both guidance terms at every diffusion step; and (iv) the proposed ADDPS, which schedules X-domain and Z-domain guidance at different steps. \textcolor{black}{Table~\ref{tab:ablation_study} shows} that the Z-only variant performs worst, confirming that Z-domain guidance is highly vulnerable to channel noise. The X-only variant is much more robust, \textcolor{black}{demonstrating} the robustness of image-domain consistency. The simultaneous X+Z variant attains the best distortion fidelity, but its FID is even worse than \textcolor{black}{that of} the X-only variant, indicating that simultaneous guidance degrades perceptual quality. In contrast, ADDPS achieves the best perceptual quality across all four metrics with only a marginal PSNR drop. These results further validate the effectiveness of the proposed ADDPS method.

\section{Conclusion}
In this paper, we have proposed a novel diffusion-based SemCom \textcolor{black}{framework} that formulates \textcolor{black}{semantic decoding} as a Bayesian inverse problem with dual-domain consistency constraints. Specifically, we first analyze why existing methods that solve \textcolor{black}{MAP estimation} fail to achieve high perceptual quality under low-SNR conditions, and explore the potential of diffusion-based posterior sampling to address this issue. We then propose \textcolor{black}{ADDPS}, which alternately enforces consistency in both the latent domain and \textcolor{black}{the} image domain during the reverse diffusion process, effectively leveraging the generative prior to produce perceptually realistic reconstructions. Extensive experiments \textcolor{black}{demonstrate} that ADDPS \textcolor{black}{outperforms} existing methods in terms of perceptual quality while maintaining competitive distortion fidelity.
{
\bibliographystyle{IEEEtran}
\bibliography{IEEEabrv,references}}

\end{document}